\documentclass{article}
\usepackage[preprint]{nips_2018}
\usepackage{color}
\usepackage{amsthm}
\newtheorem{theorem}{Theorem}[section]
\newtheorem{corollary}{Corollary}[theorem]
\newtheorem{lemma}[theorem]{Lemma}
\usepackage{times}
\usepackage{soul}
\usepackage{url}
\usepackage[hidelinks]{hyperref}
\usepackage[utf8]{inputenc}
\usepackage[small]{caption}
\usepackage{graphicx}
\usepackage{amsmath}
\usepackage{amsfonts}
\usepackage{mathrsfs}
\usepackage{booktabs}
\usepackage{algorithm}
\usepackage{algorithmic}
\usepackage[OT1]{fontenc} {}
\usepackage{subfig}
\usepackage[titletoc,title]{appendix}

\title{Nostalgic Adam: Weighting more of the past gradients when designing the adaptive learning rate}

\author{
  Haiwen Huang\\
  School of Mathematical Sciences\\
  Peking University, Beijing, 100871 \\
  \texttt{smshhw@pku.edu.cn} \\
   \And
   Chang Wang \\
  School of Mathematical Sciences\\
  Peking University, Beijing, 100871 \\
  \texttt{1500010660@pku.edu.cn} \\
   \AND
   Bin Dong \\
Beijing International Center for Mathematical Research, Peking University\\
Center for Data Science, Peking University\\
Beijing Institute of Big Data Research\\
Beijing, China\\
\texttt{dongbin@math.pku.edu.cn}}
\begin{document}

\maketitle

\begin{abstract}
First-order optimization algorithms have been proven prominent in deep learning. In particular, algorithms such as RMSProp and Adam are extremely popular. However, recent works have pointed out the lack of ``long-term memory" in Adam-like algorithms, which could hamper their performance and lead to divergence. In our study, we observe that there are benefits of weighting more of the past gradients when designing the adaptive learning rate. We therefore propose an algorithm called the Nostalgic Adam (NosAdam) with theoretically guaranteed convergence at the best known convergence rate. NosAdam can be regarded as a fix to the non-convergence issue of Adam in alternative to the recent work of \cite{j.2018on}. Our preliminary numerical experiments show that NosAdam is a promising alternative algorithm to Adam. The proofs, code and other supplementary materials can be found in an anonymously shared link\footnote{Supplementary Material:
drive.google.com/open?id=107g3qI2L
k2BGTkkXqpUTwlUfD50aj2jI}.
\end{abstract}

\section{Introduction}

Along with the rise of deep learning, various first-order stochastic optimization methods emerged. Among them, the most fundamental one is the stochastic gradient descent, and the Nesterov's Accelerated Gradient method~\cite{20001173129} is also a well-known acceleration algorithm. Recently, many adaptive stochastic optimization methods have been proposed, such as AdaGrad~\cite{Duchi:EECS-2010-24}, RMSProp~\cite{Tieleman2012}, AdaDelta~\cite{DBLP:journals/corr/abs-1212-5701} and Adam~\cite{DBLP:journals/corr/KingmaB14}. These algorithms can be written in the following general form:
\begin{equation}\label{update}
  x_{t+1} = x_t -  \frac{\alpha_t}{ \psi(g_1, ..., g_t)}\phi(g_1, ..., g_t),
\end{equation}
where $g_i$ is the gradient obtained in the $i$-th time step, $ \alpha_t/ \psi(g_1, ..., g_t)$ the adaptive learning rate, and $\phi(g_1, ..., g_t)$ the gradient estimation. There have been extensive studies on the design of gradient estimations which can be traced back to classical momentum methods ~\cite{polyak} and NAG~\cite{20001173129}. In this paper, however, we focus more on how to understand and improve the adaptive learning rate.

Adam \cite{DBLP:journals/corr/KingmaB14} is perhaps the most widely used adaptive stochastic optimization method which uses an exponential moving average (EMA) to estimate the square of the gradient scale, so that the learning rate can be adjusted adaptively. More specifically, Adam takes the form of \eqref{update} with
\begin{gather}
\psi(g_1, ..., g_t) = \sqrt{V_t}, V_t = diag(v_t)\nonumber\\
v_t = \beta_2 v_{t-1} + (1-\beta_2) g_t^2.\label{vt}
\end{gather}
We shall call $v_t$ the re-scaling term of the Adam and its variants, since it serves as a coordinate-wise re-scaling of the gradients. Despite its fast convergence and easiness in implementation, Adam is also known for its non-convergence and poor generalization in some cases \cite{j.2018on}\cite{NIPS2017_7003}. More recently, \cite{DBLP:conf/icml/BallesH18} both theoretically and empirically pointed out that generalization is mainly determined by the sign effect rather than the adaptive learning rate, and the sign effect is problem-dependent. In this paper, we are mainly dealing with the non-convergence issue and will only empirically compare generalization ability among different Adam variants.

As for the non-convergence issue, \cite{j.2018on} suggested that the EMA of $v_t$ of Adam is the cause. The main problem lies in the following quantity:
$$\Gamma_{t+1} = \frac{\sqrt{V_{t+1}}}{\alpha_{t+1}} - \frac{\sqrt{V_{t}}}{\alpha_{t}},$$
which essentially measures the change in the inverse of learning rate with respect to time. Algorithms that use EMA to estimate the scale of the gradients cannot guarantee the positive semi-definiteness of $\Gamma_t$, and that causes the non-convergence of Adam. To fix this issue, \cite{j.2018on} proposed \emph{AMSGrad}, which added one more step $\widehat v_t=\max\{\widehat v_{t-1},v_t\}$ in \eqref{vt}. AMSGrad is claimed by its authors to have a ``long-term memory" of past gradients.

Another explanation on the cause of non-convergence was recently proposed by \cite{DBLP:journals/corr/abs-1810-00143}. The authors observed that Adam may diverge because a small gradient may have a large step size which leads to a large update. Therefore, if the small $g_t$ with large step size is often in the wrong direction, it could lead to divergence. Thus, they proposed a modification to Adam called \emph{AdaShift} by replacing $g_t^2$ with $g_{t-n}^2$ for some manually chosen $n$ when calculating $v_t$. 


Both AdaShift and AMSGrad suggest that we should not fully trust the gradient information we acquire at current step, and the past gradients are useful when $g_t$ is not reliable.  In this paper, we take this idea one step further by suggesting that we may weight \emph{more} of the past gradients than the present ones. We call our algorithm Nostalgic Adam~(NosAdam). We will show that the design of the algorithm is inspired by our mathematical analysis on the convergence, and NosAdam has the fastest known convergence rate. Furthermore, we will discuss
why ``nostalgia'' is important, and empirically investigate how different designs of $v_t$ can lead to different performances from a loss landscape perspective. Finally, we examine the empirical performance of NosAdam on some common machine learning tasks. The experiments show us that NosAdam is a promising alternative to Adam and its variants.



\section{Related Work}

Adam is widely used in both academia and industry. However, it is also one of the least well-understood algorithms. In recent years, some remarkable works provided us with better understanding of the algorithm, and proposed different variants of it. Most of works focused on how to interpret or modify the re-scaling term $v_t$ of \eqref{vt}.

As mentioned above, \cite{j.2018on}, \cite{DBLP:journals/corr/abs-1810-00143} focused on the non-convergence issue of Adam, and proposed their own modified algorithms. \cite{NIPS2017_7003} pointed out the generalization issue of adaptive optimization algorithms.
Based on the assumption that $v_t$ is the estimate of the second moment estimate of $g_t$, \cite{DBLP:conf/icml/BallesH18} dissected Adam into sign-based direction and variance adaption magnitude. They also pointed out that the sign-based direction part is the decisive factor of generalization performance, and that is problem-dependent. This in a way addressed the generalization issue raised in \cite{NIPS2017_7003}.

However, the interpretation of $v_t$ as an estimate of the second moment assumption may not be correct, since \cite{chen2019padam} showed that $v_t^{1/2}$ in the Adam update \eqref{vt} can be replaced by $v_t^p$ for any $p \in (0, \frac{1}{2}]$. The modified algorithm is called Padam. In our supplementary material, we also proved that a convergence theorem of a ``p-norm'' form of NosAdam, where the re-scaling term $v_t$ can be essentially viewed as a ``p-moment'' of $g_t$. These discoveries cast doubts on the the second moment assumption, since both the convergence analysis and empirical performance seemed not so dependent on this assumption.

The true role of $v_t$, however, remains a mystery. In AdaGrad \cite{Duchi:EECS-2010-24}, which is a special case of NosAdam, the authors mentioned an metaphor that ``the adaptation allows us to find needles in haystacks in the form of very predictive but rarely seen features.'' They were suggesting that $v_t$ is to some extent balancing the update speeds of different features according to their abundance in the data set. This understanding might be supported by a previous work called SWATS (Switching from Adam to SGD) \cite{DBLP:journals/corr/abs-1712-07628}, which uses Adam for earlier epochs and then fix the re-scaling term $v_t$ for later epochs. This suggests that there may be some sort of “optimal” re-scaling term, and we can keep using it after we obtain a good enough estimate.

Despite all the previous efforts, our understanding of the re-scaling term $v_t$ is still very limited. In this paper, we investigate the issue from a loss landscape approach, and this provides us with some deeper understanding of when and how different Adam-like algorithms can perform well or poorly.

\section{Nostalgic Adam Algorithm}

In this section, we introduce the Nostalgic Adam (NosAdam) algorithm, followed by a discussion on its convergence. Let us first consider a  general situation where we allow the parameter $\beta_2$ in Equation~\eqref{vt} change in time $t$. Without loss of generality, we may let $\beta_{2,t} = \frac{B_{t-1}}{B_t}$. Then, the NosAdam algorithm reads as in Algorithm \ref{nosadam}. Like Adam and its variants, the condition $\Gamma_t \succeq 0$ is crucial in ensuring convergence. We will also see that to ensure positive semi-definiteness of $\Gamma_t$, the algorithm naturally requires to weight more of the past gradients than the more recent ones when calculating $v_t$. To see this, we first present the following lemma.

\begin{lemma}
The positive semi-definiteness of $\frac{V_t}{\alpha^2_t}-\frac{V_{t-1}}{\alpha^2_{t-1}}$ is tightly satisfied if $\frac{B_t}{t}$ is non-increasing.
\end{lemma}

\begin{proof}
  \begin{align*}
  \frac{V_t}{\alpha_t^2}
  =& \frac{t}{\alpha^2}\sum_{j=1}^t \Pi_{k=1}^{t-j} \beta_{2,t-k+1}(1-\beta_{2,j})g_j^2\\
  =& \frac{t}{\alpha^2} \sum_{j=1}^t \frac{B_{t-1}}{B_t} \ldots \frac{B_{j}}{B_{j+1}}\frac{B_{j}-B_{j-1}}{B_{j}} g_j^2 \\
  =& \frac{t}{B_t \alpha^2} \sum_{j=1}^t b_j g_j^2
  \geq \frac{t-1}{B_{t-1} \alpha^2} \sum_{j=1}^{t-1} b_j g_j^2\\
  =&  \frac{V_{t-1}}{\alpha_{t-1}^2}
  \end{align*}

\end{proof}

Here the ``tightly satified'' in the lemma means the conclusion cannot be strengthened, in that if $\frac{B_t}{t}$ is increasing, then $\frac{V_t}{\alpha^2_t}-\frac{V_{t-1}}{\alpha^2_{t-1}}$ will be very easily violated since $g_j$ can be infinitesimal.

Again, without loss of generality, we can write $B_t$ as $\sum_{j=1}^t b_j$. Then, it is not hard to see that $\frac{B_t}{t}$ is non-increasing if and only if $b_j$ is non-increasing. Noting that $v_t = \sum_{k=1}^t g_k^2 \frac{b_k}{B_t}$, we can see that the sufficient condition for positive semi-definiteness of $\Gamma_t$ is that \emph{in the weighted average $v_t$, the weights of gradients should be non-increasing w.r.t. $t$}. In other words, we should weight more of the past gradients than the more recent ones.

\begin{algorithm}[t]
\caption{Nostalgic Adam Algorithm}
\label{nosadam}
\textbf{Input}: $x \in F$, $m_0 = 0$, $V_0 = 0$
\begin{algorithmic}[1]
\FOR{$t = 1 $ \textbf{to} $T$}
    \STATE $g_t = \nabla f_t(x_t)$
    \STATE $\beta_{2,t} = B_{t-1}/B_t$, where $B_t = \sum_{k=1}^t b_k$ for $t \geq 1$, $b_k\ge0$ and $B_0 = 0$
    \STATE $m_t = \beta_1 m_{t-1} + (1-\beta_1)g_t$
    \STATE $V_t = \beta_{2,t}V_{t-1} + (1-\beta_{2,t})g_t^2$
    \STATE $\hat{x}_{t+1} = x_t - \alpha_t m_t / \sqrt{V_t}$
    \STATE $x_{t+1} = \mathcal{P}_{\mathcal{F}, \sqrt{V_t}}(\hat{x}_{t+1})$
\ENDFOR
\end{algorithmic}
\end{algorithm}

From Algorithm \ref{nosadam}, we can see that $v_t$ can either decrease or increase based on the relationship between $v_{t-1}$ and $g_t^2$, which is the reason why NosAdam circumvents the flaw of AMSGrad (Figure \ref{amsgrad_trap}). Convergence of NosAdam is also guaranteed as stated by the following theorem.

\begin{theorem}[Convergence of NosAdam]\label{thm:nosadam}
Let $B_t$ and $b_k$ be the sequences defined in Algorithm \ref{nosadam}, $\alpha_t = \alpha/\sqrt{t}$, $\beta_{1,1} = \beta_1, \beta_{1,t} \leq \beta_1$ for all t. Assume that $\mathscr{F}$ has bounded diameter $D_{\infty}$ and $||\nabla f_t(x)||_{\infty} \leq G_{\infty}$ for all t and $x \in \mathscr{F}$. Furthermore, let $\beta_{2,t}$ be such that the following conditions are satisfied:
\begin{align*}\label{constraints}
  1.& \frac{B_t}{t} \leq \frac{B_{t-1}}{t-1}\\
  2.& \frac{B_t}{t b_t^2} \geq \frac{B_{t-1}}{(t-1) b_{t-1}^2}
\end{align*}

Then for $\{x_t\}$ generated using NosAdam, we have the following bound on the regret
\begin{gather*}
R_T \leq \frac{D_{\infty}^2}{2\alpha(1-\beta_1)} \sum_{i=1}^d \sqrt{T}v_{T,i}^{\frac{1}{2}} +
     \frac{D_{\infty}^2}{2(1-\beta_1)}\sum_{t=1}^T\sum_{i=1}^d \frac{\beta_{1,t}v_{t,i}^{\frac{1}{2}}}{\alpha_t} \\
     + \frac{\alpha \beta_1}{(1-\beta_1)^3} \sum_{i=1}^d\sqrt{\frac{B_T}{T}\frac{\sum_{t=1}^T b_t g_{t,i}^2 }{b_T^2}}
\end{gather*}
\end{theorem}

Here, we have adopted the notation of online optimization introduced in \cite{Zinkevich:2003:OCP:3041838.3041955}. At each time step $t$, the optimization algorithm picks a point $x_t$ in its feasible set $\mathcal{F} \in \mathbb{R}^d$. Let $f_t$ be the loss function corresponding to the underlying mini-batch, and the algorithm incurs loss $f_t(x_t)$. We evaluate our algorithm using the regret that is defined as the sum of all the previous differences between the online prediction $f_t(x_t)$ and loss incurred by the fixed parameter point in $\mathcal{F}$ for all the previous steps, i.e.
\begin{equation}\label{regret}
  R_T = \sum_{t=1}^T f_t(x_t) - \min_{x\in \mathcal{F}} \sum_{t=1}^T f_t(x).
\end{equation}
Denote $\mathcal{S}_{d}^{+}$ the set of all positive definite $d \times d$ matrices. The projection operator $\mathcal{P}_{\mathcal{F}, A(y)}$ for $A \in \mathcal{S}_d^+$ is defined as $\mathrm{argmin}_{x \in \mathcal{F}} ||A^{1/2}(x-y)||$ for $ y \in \mathbb{R}^d$. Finally, we say $\mathcal{F}$ has bounded diameter $D_{\infty}$ if $||x-y||_{\infty} \leq D_{\infty}$ for all $x, y \in \mathcal{F}$.

One notable characteristic of NosAdam, which makes it rather different from the analysis by \cite{j.2018on}, is that the conditions on $B_t$ and $b_t$ are data-independent and are very easy to check. In particular, if we choose $B_t$ as a hyperharmonic series, i.e. $B_t = \sum_{k=1}^t k^{-\gamma}$, then the convergence criteria are automatically satisfied. We shall call this special case NosAdam-HH, and its convergence result is summarized in the following corollary.

\begin{corollary}\label{corollary}
Suppose$\beta_{1,t} = \beta_1 \lambda^{t-1}$,  $b_k = k^{-\gamma}, \gamma \geq0$ , thus $B_t = \sum_{k=1}^{t} k^{-\gamma}$, and $\beta_{2,t} = B_{t-1}/B_t < 1$ in Algorithm \ref{nosadam}. Then $B_t$ and $b_t$ satisfy the constraints in Therorem \ref{thm:nosadam}, and we have
\begin{gather*}
R_T \leq \frac{D_{\infty}^2}{2\alpha(1-\beta_1)} \sum_{i=1}^d \sqrt{T}v_{T,i}^{\frac{1}{2}} +
        \frac{D_{\infty}^2G_{\infty}\beta_1}{2(1-\beta_1)}\frac{1}{(1-\lambda)^2}\cdot d \\
     + \frac{2\alpha \beta_1}{(1-\beta_1)^3} G_{\infty} \sqrt{T}
\end{gather*}
\end{corollary}

Our theory shows that the proposed NosAdam achieves convergence rate of $O(1/\sqrt{T})$, which is so far the best known convergence rate.

\section{Why Nostalgic?}
In this section, we investigate more about the mechanism behind Adam and AMSGrad, and analyze the pros and cons of being ``nostalgic''.

As mentioned in Section 1, \cite{j.2018on} proved that if $\Gamma_t$ is positive semi-definite, Adam converges. Otherwise, it may diverge. An example of divergence made by \cite{j.2018on} is
\begin{equation}\label{fx}
  f_t(x)=
  \begin{cases}
  Cx& \text{t\; mod\; 3=1}\\
  -x& \text{otherwise }
  \end{cases},
\end{equation}
where $C$ is slightly larger than 2. The correct optimization direction should be -1, while Adam would go towards 1. To fix this, they proposed AMSGrad, which ensures $\Gamma_t \succeq 0$ by updating $v_t$ as follows
\begin{gather*}
 v_t = \beta_2 v_{t-1} + (1-\beta_2) g_t^2,\\
 \hat{v}_t = \text{max}(\hat{v}_{t-1}, v_t),
\end{gather*}
where $\hat{v}_t$ is used in the update step.

However, this example is not representative of real situations. Also, the explanation of ``long-term memory" by \cite{j.2018on} is not very illustrative. In the remaining part of this section, we aim to discuss some more realistic senarios and try to understand the pros and cons of different algorithms.

\begin{figure}[tb]
\centering
\includegraphics[width=0.6\textwidth]{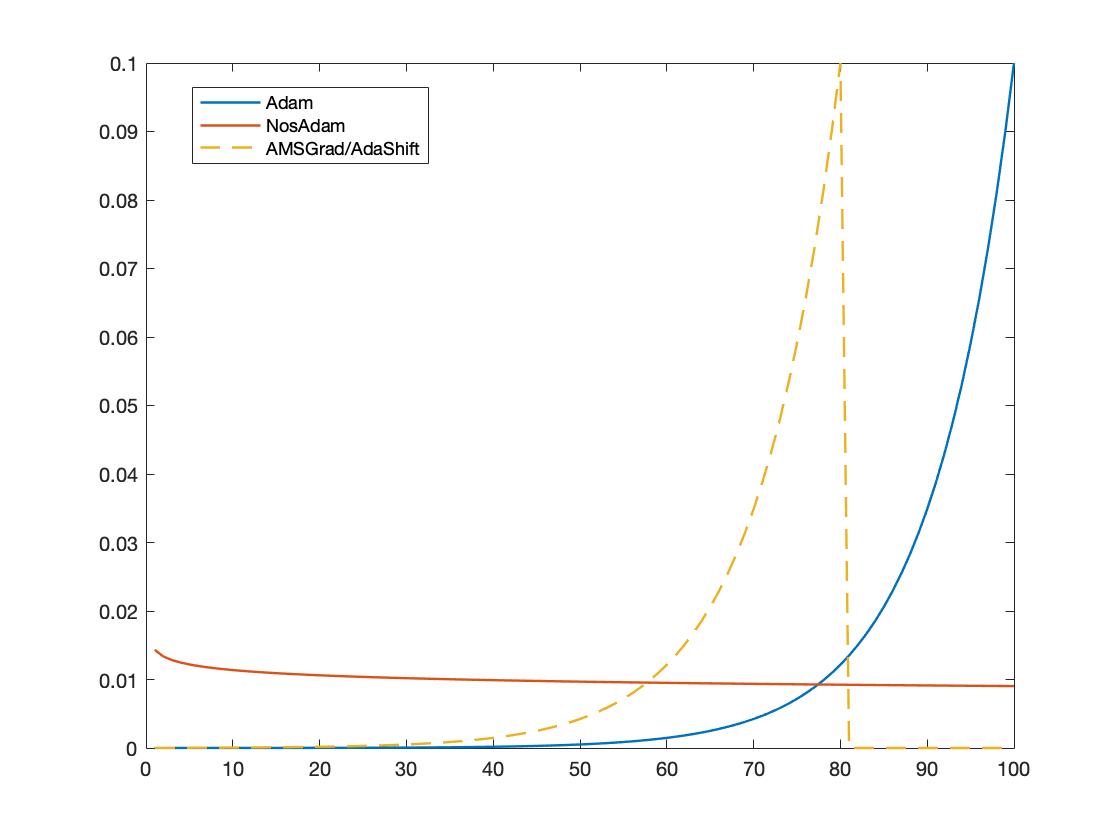}
\caption{Weight comparison among Adam, NosAdam and AMSGrad. The $y$-axis shows the weight of each step, and the $x$-axis shows the number of steps }
\label{weight_compare}
\end{figure}

We start from analyzing the different weighting strategies when calculating $v_t$.
For Adam, $$v_t^{(\text{Adam})} = \sum_{k=1}^{\infty} (1-\beta_2)\beta_2^{t-k} g_k^2,$$ and the weight $(1-\beta_2)\beta_2^{t-k}$ increases exponentially.
For NosAdam, $$v_t^{(\text{NosAdam})} = \sum_{k=1}^t g_k^2 \frac{b_k}{B_t},$$ and for NosAdam-HH, $b_k = k^{-\gamma}$ is the $k$-th term of a hyperharmonic series.
For AMSGrad, $v_t^{(\text{AMSGrad})}$ is data-dependent and therefore cannot be explicitly expressed. However, $v_t^{(\text{AMSGrad})}$ is chosen to be the largest in $\{v_s^{(\text{Adam})}: 0\le s\le t\}$. Therefore, it can be seen as a shifted version of $v_t^{(\text{Adam})}$, i.e. $v_s^{(\text{Adam})}=v_{t-n}^{(\text{Adam})}$, where $n$ depends on the data. This is similar as AdaShift, where $n$ is instead a hyperparameter. Figure \ref{weight_compare} plots the first 100 weights of Adam, NosAdam and AMSGrad, where $\beta_2$, $\gamma$, $n$, is chosen as 0.9, 0.1 and 20, respectively.

From the above analysis, we can see that $v_t$ of Adam is mainly determined by its most current gradients. Therefore, when $g_t$ keeps being small, the adaptive learning rate could be large, which may lead to oscillation of the sequence, and increasing chance of being trapped in local minimum. On the other hand, NosAdam adopts a more stable calculation of $v_t$, since it relies on all the past gradients.

We support the above discussion with an example of an objective function with a bowl-shaped landscape where the global minima is at the bottom of the bowl with lots of local minimum surrounding it. The explicit formula of the objective function is
\begin{align*}
f(x,&y,z) = -ae^{-b((x-\pi)^2 + (y-\pi)^2) + (z-\pi)^2)} \\
&- c \sum_i \text{cos}(x)\text{cos}(y)e^{-\beta ((x-r\text{sin}(\frac{i}{2})-\pi)^2+(y-r\text{cos}(\frac{i}{2})-\pi)^2)}.
\end{align*}\label{example}
Figure \ref{adam_diverge}a shows one slice of the function for $z=2.34$. In the function, $a$ and $b$ determine the depth and width of the global minima, and $c$, $r$, $\beta$ determine depth, location and width of the local minimums. In this example, $a$, $b$, $c$, $r$, $\beta$ are set to 30, 0.007, 0.25, 1, 20, respectively.
\begin{figure}[tb]
  \centering
  \subfloat[Bowl-shaped Landscape]{\label{1}\includegraphics[width=0.5\textwidth]{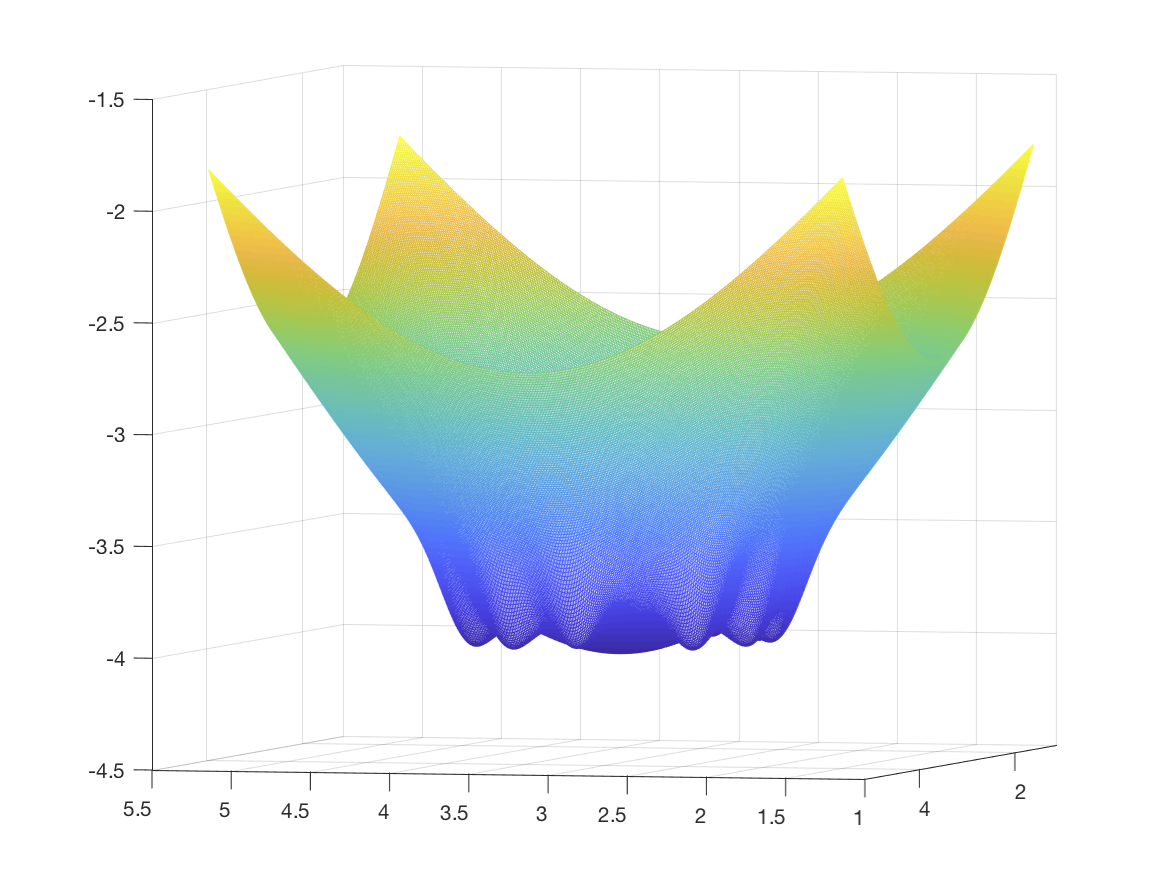}}
  \subfloat[Trajectories of Adam and NosAdam]{\label{2}\includegraphics[width=0.5\textwidth]{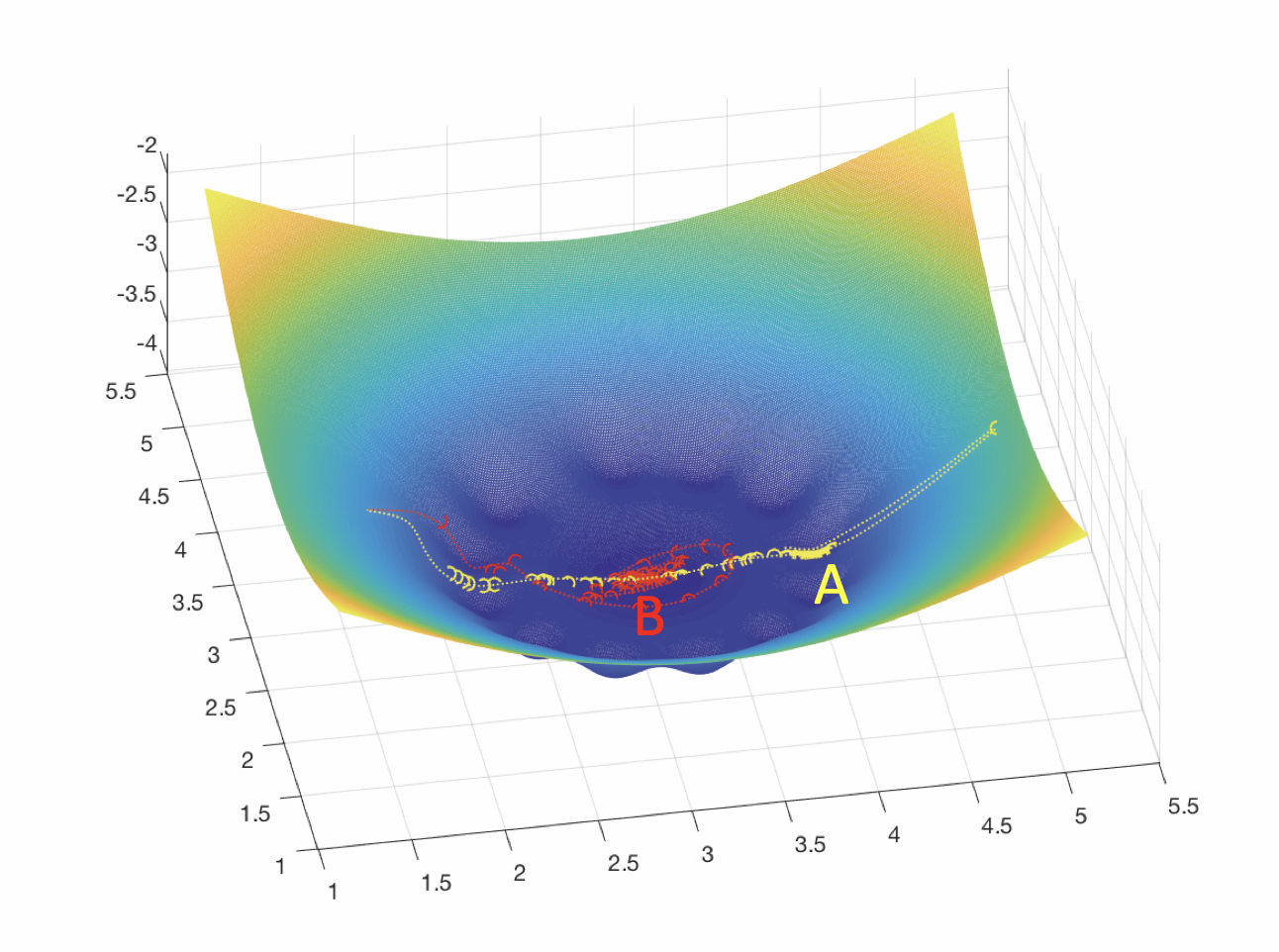}}
\caption{Example of an objective function with a bowl-shaped landscape. It has a wide global minima and some sharp local minimum surrounding it. In \ref{2}, the red trajectory is the sequence generated by NosAdam and yellow trajectory by Adam. The trajectory of Adam ends up in valley $A$ and NosAdam in valley $B$. This shows that Adam would easily diverge due to unstable calculations of $v_t$.}
\label{adam_diverge}
\end{figure}

Figure \ref{adam_diverge}b shows different trajectories of Adam and NosAdam when they are initiated at the same point on the side of the bowl. As expected, the trajectory of Adam (yellow) passes the global minima and ends up trapped in valley $A$, while NosAdam (red) gradually converges to the global minima, i.e. valley $B$.
\begin{figure}[tb]
\centering
\includegraphics[width=0.7\textwidth]{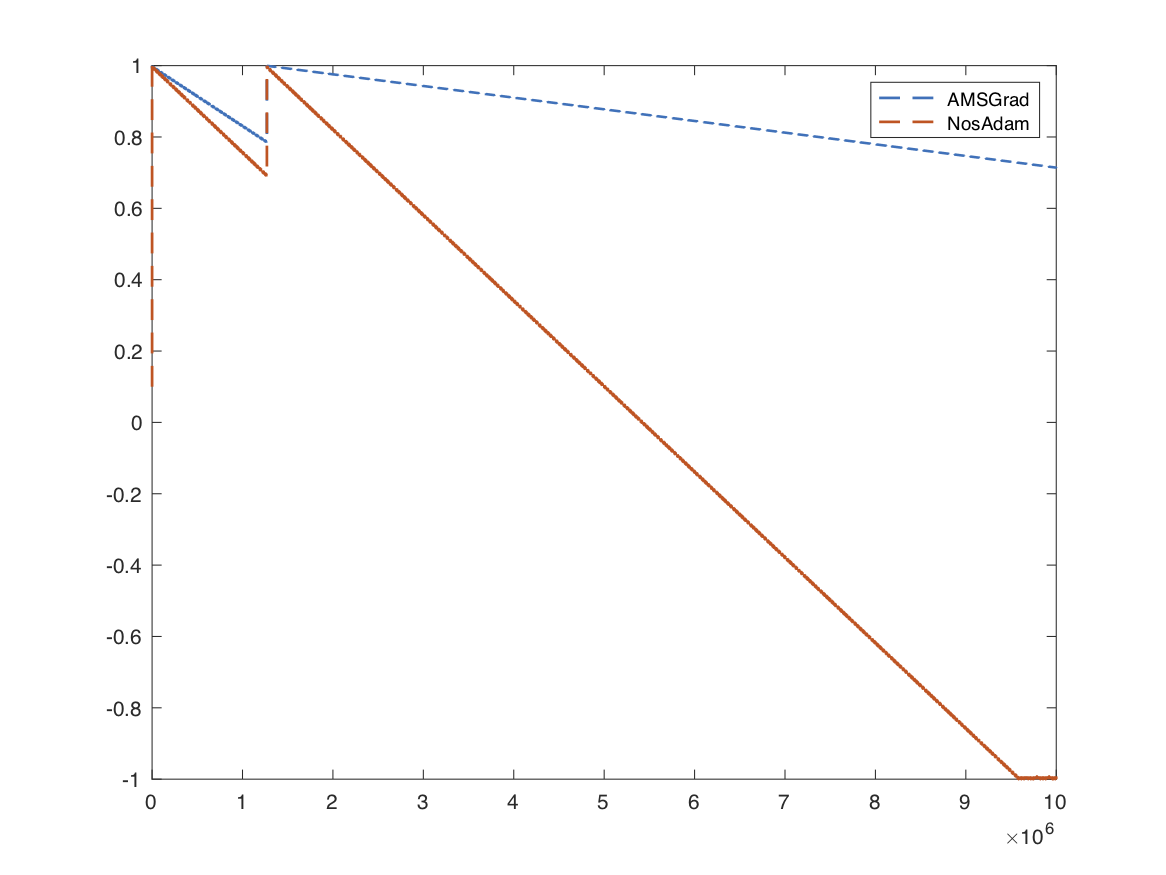}
\caption{Appearance of a large gradient at around $10^6$ step. The $y$-axis shows the value $x$, and the $x$-axis shows the number of iterations. The figure shows AMSGrad is greatly slowed down after encountering a large gradient.}
\label{amsgrad_slow}
\end{figure}
\begin{figure}[tb]
  \centering

  \subfloat[Sharper Minima]{\label{11}\includegraphics[width=0.5\textwidth]{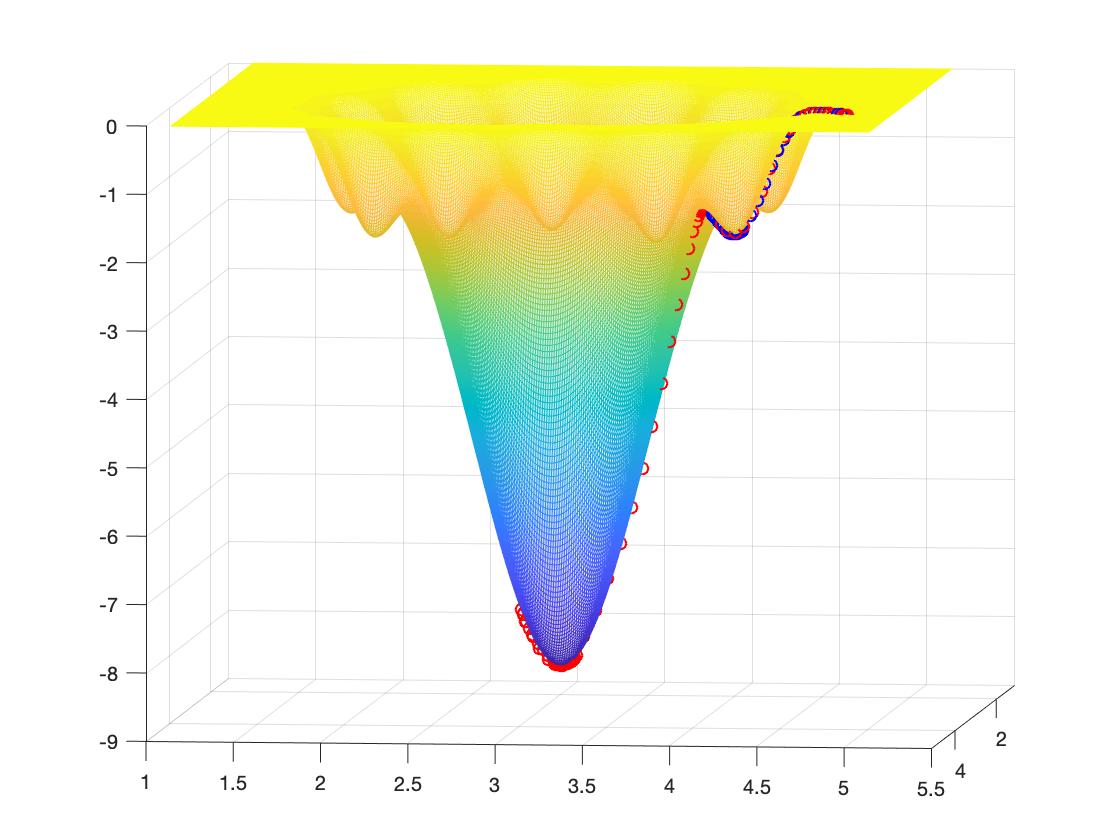}}
  \subfloat[Trajectories of AMSGrad and NosAdam]{\label{22}\includegraphics[width=0.5\textwidth]{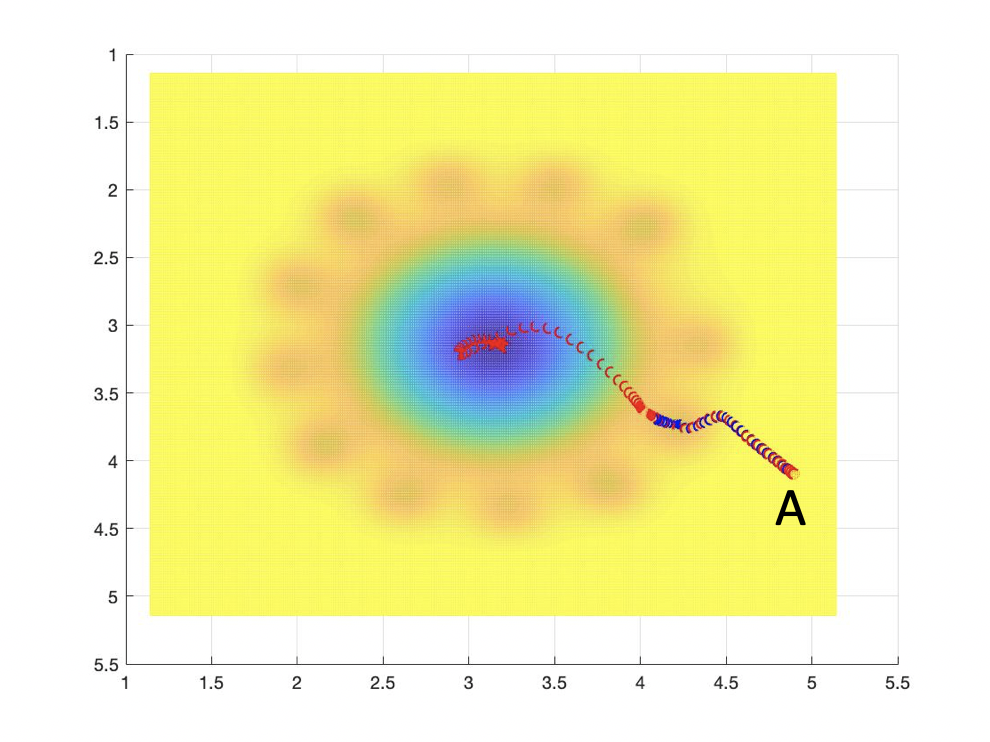}}

\caption{Figure \ref{11} shows the landscape of the objective function. Figure \ref{22} shows the different behaviours of AMSGrad and NosAdam, with the sequence generated by AMSGrad colored in blue and NosAdam in red, and they are initiated at location A.}
\label{amsgrad_trap}
\end{figure}

With the operation of taking max, AMSGrad does not have the same non-convergence issue as discussed above.
However, taking max may be problematic as well since \emph{$\hat{v}_t$ can never decrease}. If a very large gradient appears at an iteration, then the adaptive learning rate for all later steps will keep being small. For example, if a large gradient (e.g. 100 times of the original gradient) appears at the $10^6$ step in the example \eqref{fx}, we can see that AMSGrad will converge very slowly. This, however, will not be a problem for NosAdam which has the ability of both increasing and decreasing its $v_t$. See Figure \ref{amsgrad_slow} for a demonstration. Another example with sharper local minima by setting $b=2$, $c=4$, $r=1.3$ is given in Figure \ref{amsgrad_trap}; and the algorithms are initialized at location A. One can see that the sequence generated by AMSGrad is trapped in a sharp local minima, whereas NosAdam still converges to the global minimum. From these examples we can see that the operation of taking max of AMSGrad has some intrinsic flaws though it promises convergence. The way of computing $v_t$ in NosAdam seems superior.

There are also situations in which NosAdam can work poorly. Just because NosAdam is nostalgic, it requires a relatively good initial point to achieve good performances though this is commonly required by most optimization algorithms. However, Adam can be less affected by bad initializations sometime due to its specific way of calculating $v_t$. This gives it a chance of jumping out of the local minimum (and a chance of jumping out of the global minima as well as shown in Figure \ref{adam_diverge}). To demonstrate this,  we let both Adam and NosAdma initialize in the valley A (see Figure \ref{nosadam_trap}). We can see that the trajectory of Adam manages to jump out of the valley, while it is more difficult for NosAdam to do so.

We note that although NosAdam requires good initialization, it does not necessarily mean initializing near the global minima. Since the algorithm is nostalgic, as long as the initial gradients are pointing towards the right direction, the algorithm may still converge to the global minima even though the initialization is far away from the global minima. As we can see from Figure \ref{amsgrad_trap} that NosAdam converges because all of the gradients are good ones at the beginning of the algorithm, which generates enough momentum to help the sequence dashes through the region with sharp local minimum.

Like any Adam-like algorithm, the convergence of NosAdam depends on the loss landscape and initialization. However, if the landscape is as shown in the above figures, then NosAdam has a better chance to converge than Adam and AMSGrad. In practice, it is therefore helpful to first examine the loss landscape before selecting an algorithm. However, it is time consuming to do in general. Nonetheless, earlier studies showed that neural networks with skip connections like ResNet and DenseNet lead to coercive loss functions similar to the one shown in the above figures \cite{visualloss}.

\begin{figure}[tb]
  \centering

  \subfloat[Sharper Minima]{\label{111}\includegraphics[width=0.5\textwidth]{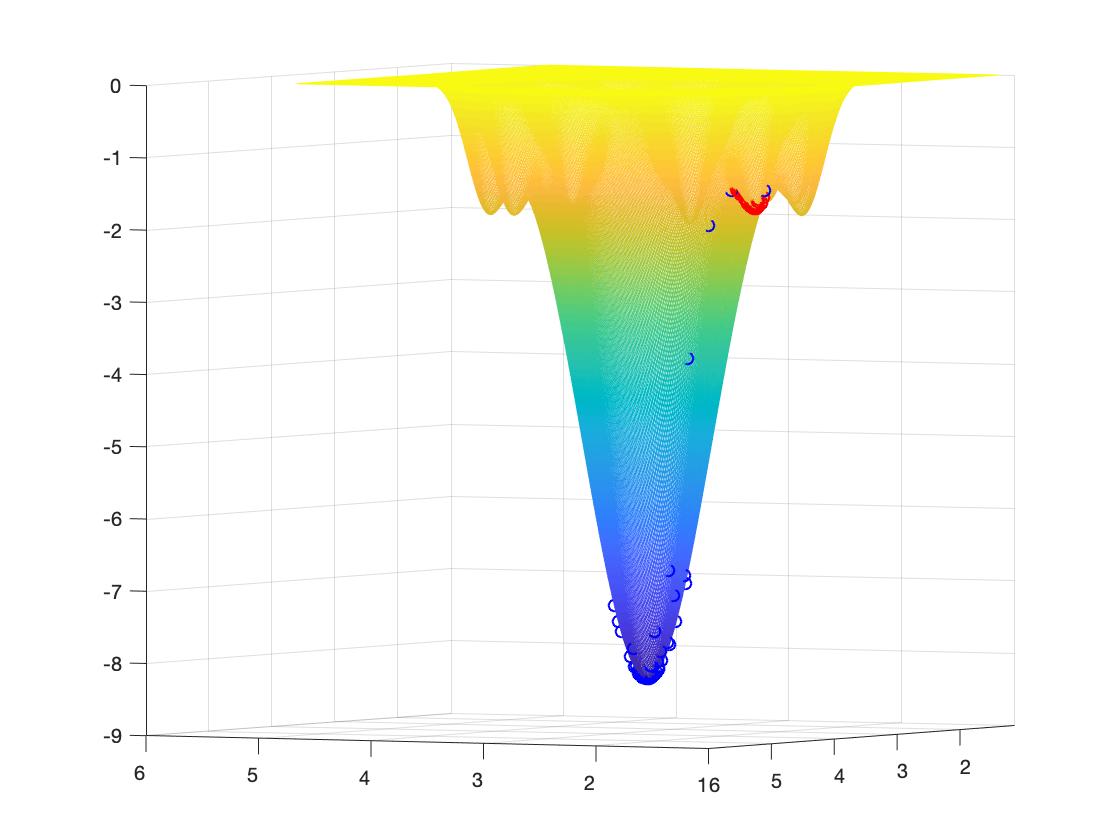}}
  \subfloat[Trajectories of Adam and NosAdam]{\label{222}\includegraphics[width=0.5\textwidth]{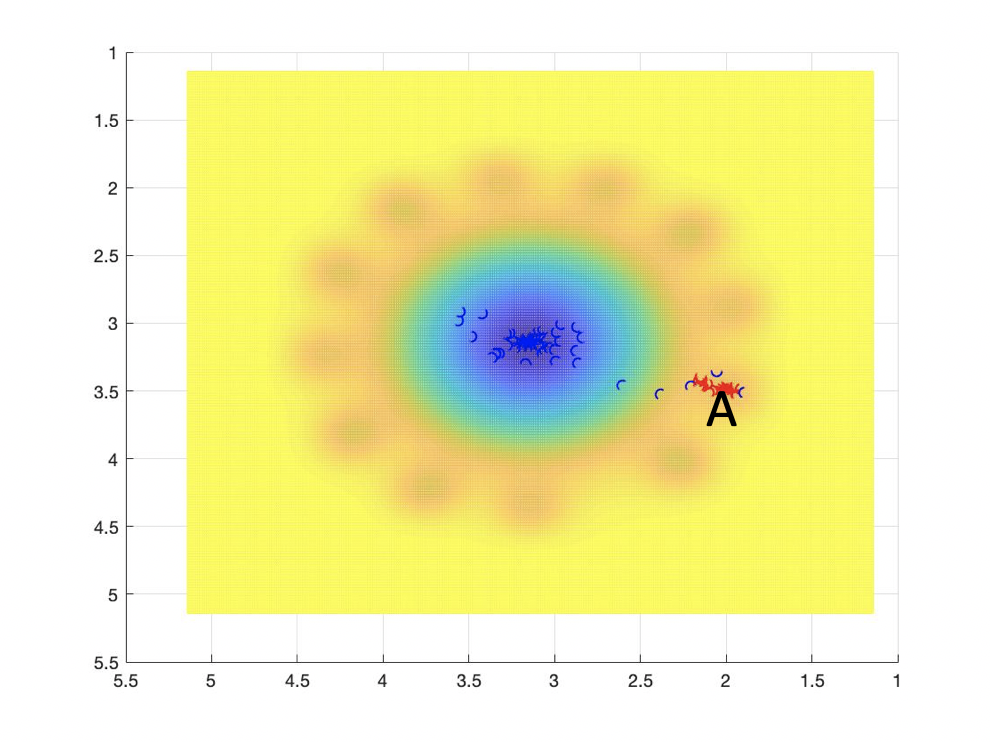}}

\caption{Figure \ref{111} shows the landscape of the objective function. Figure \ref{222} shows the different behaviours of Adam and NosAdam, with the sequence generated by Adam colored in blue and NosAdam in red, and they are initiated in the valley A.}
\label{nosadam_trap}
\end{figure}

\section{Experiments}

In this section, we conduct some experiments to compare NosAdam with Adam and its variant AMSGrad. We consider the task of multi-class classification using logistic regression, multi-layer fully connected neural networks and deep convolutional neural networks on MNIST \cite{10027939599} and CIFAR-10 \cite{cifar10}. The results generally indicate that NosAdam is a promising algorithm that works well in practice.

Throughout our experiments, we fixed $\beta_1$ to be 0.9, $\beta_2$ to be 0.999 for Adam and AMSGrad, and search $\gamma$ in $ \{\text{1e-1},\text{1e-2},\text{1e-3}, \text{1e-4}\}$ for NosAdam. The initial learning rate is chosen from $\{\text{1e-3}, \text{2e-3}, ..., \text{9e-3}, \text{1e-2}, \text{2e-2}, ..., \text{9e-2}, \text{1e-1}, \text{2e-1}, ..., \text{9e-1}\}$ and the results are reported using the best set of hyper-parameters. All the experiments are done using Pytorch0.4.

\textbf{Logistic Regression}
To investigate the performance of the algorithms on convex problems, we evaluate Adam, AMSGrad and NosAdam on multi-class logistic regression problem using the MNIST dataset. To be consistent with the theory, we set the step size $\alpha_t = \alpha/\sqrt{t}$. We set the minibatch size to be 128. According to Figure \ref{logis}, the three algorithms have very similar performance.

\textbf{Multilayer Fully Connected Neural Networks}

We first train a simple fully connected neural network with 1 hidden layer (with 100 neurons and ReLU as the activation function) for the multi-class classification problem on MNIST. We use constant step size $\alpha_t =\alpha$ throughout the experiments for this set of experiments. The results are shown in Figure \ref{multi}. We can see that NosAdam slightly outperforms AMSGrad, while Adam is much worse than both NosAdam and AMSGrad and oscillates a lot. This is due to the difference of the definition of $v_t$ for each algorithm: $v_t$ in AMSGrad and NosAdam gradually becomes stationary and stays at a good re-scaling value; while $v_t$ in Adam does not have such property.

\textbf{Deep Convolutional Neural Networks}
Finally, we train a deep convolutional neural network on CIFAR-10. Wide Residual Network \cite{DBLP:journals/corr/ZagoruykoK16} is known to be able to achieve high accuracy with much less layers than ResNet \cite{DBLP:journals/corr/HeZRS15}. In our experiment, we choose Wide ResNet28. The model is trained on 4 GPUs with the minibatch size 100. The initial learning rate is decayed at epoch 50 and epoch 100 by multiplying 0.1. In our experiments, the optimal performances are usually achieved when the learning rate is around 0.02 for all the three algorithms. For reproducibility, an anonymous link of code will be provided in the supplementary material.

Our results are shown in Figure \ref{exp2}. We observe that NosAdam works slightly better than AMSGrad and Adam in terms of both convergence speed and generalization. This indicates that NosAdam is a promising alternative to Adam and its variants.

\begin{figure}[tb]
  \centering

  \subfloat[Logistic Regression]{\label{logis}\includegraphics[width=0.5\textwidth]{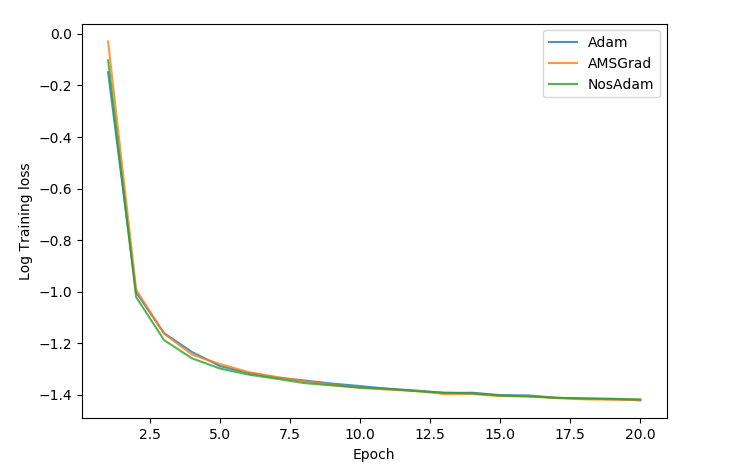}}
  \subfloat[Multi-layer Fully Connected Neural Network]{\label{multi}\includegraphics[width=0.5\textwidth]{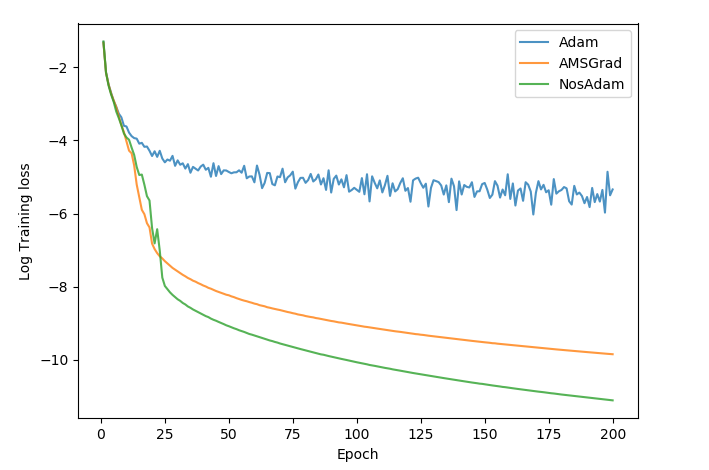}}

\caption{Experiments of logistic regression and multi-layer fully connected neural network on MNIST.}
\label{exp1}
\end{figure}

\begin{figure}[tb]
  \centering

  \subfloat[Log Training Loss]{\label{figur:1}\includegraphics[width=0.5\textwidth]{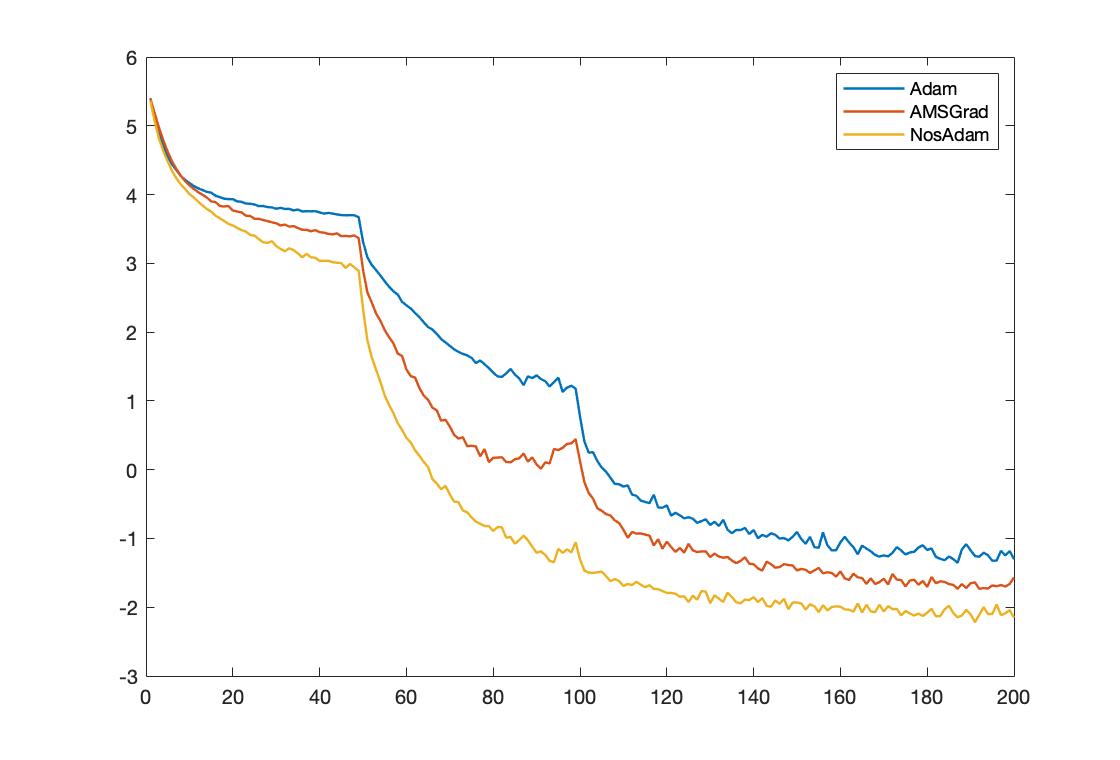}}
  \subfloat[Test Accuracy]{\label{figur:2}\includegraphics[width=0.5\textwidth]{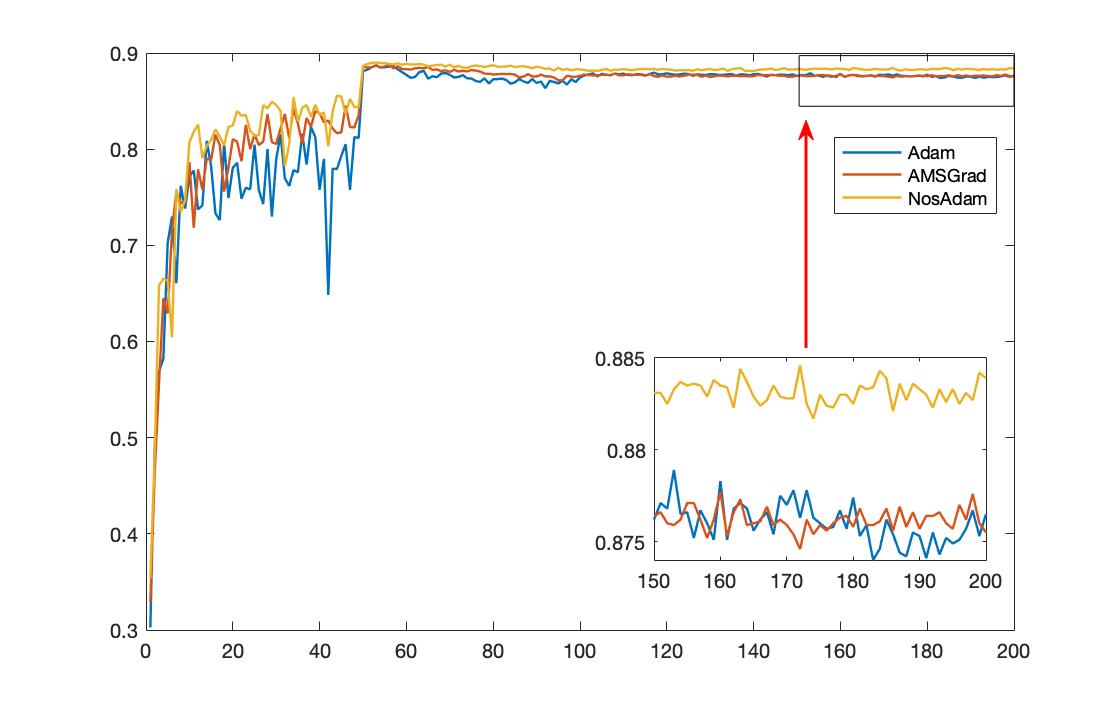}}

\caption{Experiments of Wide ResNet on CIFAR-10.}
\label{exp2}
\end{figure}
\section{Discussion}

In this paper, we suggested that we should weight more of the past gradients when designing the adaptive learning rate. In fact, our original intuition came from mathematical analysis of the convergence of Adam-like algorithms. Based on such observation, we then proposed a new algorithm called Nostalgic Adam (NosAdam), and provided a convergence analysis. We also discussed the pros and cons of NosAdam comparing to Adam and AMSGrad using a simple example, which gave us a better idea when NosAdam could be effective.

For future works, we believe that loss landscape analysis and the design of a strategy to choose different algorithms adaptively based on the loss landscape would be worth pursuing. Hopefully, we can design an optimization algorithm that can adaptively adjust its re-scaling term in order to fully exploit the local geometry of the loss landscape.

\section*{Acknowledgments}
This work would not have existed without the support of BICMR and School of Mathematical
Sciences, Peking University. Bin Dong is supported in part by  Beijing Natural Science Foundation (Z180001).

\newpage

\begin{appendices}


\section{Convergence of p-NosAdam}

In this appendix, we use the same notations as in the paper ``Nostalgic Adam: Weighting more of the past gradients when designing the adaptive learning rate''. We are going to prove a more general convergence theorem. In the original paper, we propose NosAdam, as shown in Algorithm \ref{nosadam}. But in fact, NosAdam can be considered as a particular case of a more general algorithm, in which we replaces $g_t^2$ in the calculation of $v_t$ by $g_p$, and $v_t^{1/2}$ in the update equation by $v_t^{1/p}$. We call this algorithm p-NosAdam, as shown in Algorithm \ref{pnosadam}. NosAdam is p-NosAdam when $p=2$.

In the remaining part of this appendix, we are going to prove the convergence theorem of p-NosAdam when $p>1$. From Theorem \ref{thm:nosadam}, we can see that the regret bound is $O(T^{\text{max}(\frac{1}{p}, \frac{p-1}{p})})$.


\begin{algorithm}
\caption{p-NosAdam Algorithm}
\label{pnosadam}
\textbf{Input}: $x \in F$, $m_0 = 0$, $V_0 = 0$
\begin{algorithmic}[1]
\FOR{$t = 1 $ \textbf{to} $T$}
    \STATE $g_t = \nabla f_t(x_t)$
    \STATE $\beta_{2,t} = B_{t-1}/B_t$, where $B_t = \sum_{k=1}^t b_k$ for $t \geq 1$, $b_k\ge0$ and $B_0 = 0$
    \STATE $m_t = \beta_1 m_{t-1} + (1-\beta_1)g_t$
    \STATE $V_t = \beta_{2,t}V_{t-1} + (1-\beta_{2,t})g_t^2$
    \STATE $\hat{x}_{t+1} = x_t - \alpha_t m_t / \sqrt{V_t}$
    \STATE $x_{t+1} = \mathcal{P}_{\mathcal{F}, \sqrt{V_t}}(\hat{x}_{t+1})$
\ENDFOR
\end{algorithmic}
\end{algorithm}

\begin{theorem}[Convergence of p-NosAdam]\label{thm:nosadam}
  Let $B_t$ and $b_k$ be the sequences defined in p-NosAdam, $\alpha_t=\alpha/t^{1/p}, p>1$, $\beta_{1,1} = \beta_1, \beta_{1,t} \leq \beta_1$ for all t. Assume that $\mathscr{F}$ has bounded diameter $D_{\infty}$ and $||\nabla f_t(x)||_{\infty} \leq G_{\infty}$ for all t and $x \in \mathscr{F}$. Furthermore, let $\beta_{2,t}$ be such that the following conditions are satisfied:
  \begin{align*}\label{constraints}
  1.& \frac{B_t}{t} \leq \frac{B_{t-1}}{t-1}\\
  2.& \frac{B_t}{t b_t^p} \geq \frac{B_{t-1}}{(t-1) b_{t-1}^p}
  \end{align*}
  
  Then for $\{x_t\}$ generated using p-NosAdam, we have the following bound on the regret
  
  $$R_T \leq \frac{D_{\infty}^2}{2\alpha(1-\beta_1)} \sum_{i=1}^d T^{\frac{1}{p}}v_{T,i}^{\frac{1}{p}} +
  \frac{D_{\infty}^2}{2(1-\beta_1)}\sum_{t=1}^T\sum_{i=1}^d \frac{\beta_{1,t}v_{t,i}^{\frac{1}{p}}}{\alpha_t}
  + \frac{\alpha (\beta_1+1)}{(1-\beta_1)^3} \sum_{i=1}^d  (\sum_{t=1}^T b_t g_{t,i}^p)^{\frac{p-1}{p}}(\frac{B_T}{T b_T^p})^{\frac{1}{p}} $$
  
\end{theorem}

\subsection*{Proof of Theorem \ref{thm:nosadam}:}
Recall that \begin{equation}\label{regret}
  R_T = \sum_{t=1}^T f_t(x_t) - \min_{x\in \mathcal{F}} \sum_{t=1}^T f_t(x).
\end{equation}
Let $x^* = \text{argmin}_{x\in \mathcal{F}}\sum_{t=1}^T f_t(x)$. Therefore $R_T = \sum_{t=1}^{T} f_t(x_t)-f_t(x^*) $.

To prove this theorem, we will use the following lemmas.

\begin{lemma}\label{all}
  \begin{align*}
  \sum_{t=1}^{T} f_t(x_t)-f_t(x^*)&\leq \sum_{t=1}^{T} [\frac{1}{2\alpha_t(1-\beta_{1t})}( ||V_t^{1/2p}(x_t-x^*)||^2\\
 &\hspace*{0.2in}-||V_t^{1/2p}(x_{t+1}-x^*)||^2)+\frac{\alpha_t}{2(1-\beta_{1t})}||V_t^{-1/2p}m_t||^2\\
  &\hspace*{0.4in}+\frac{\beta_{1t}}{2(1-\beta_{1t})}\alpha||V_t^{-1/2p}m_{t+1}||^2 +\frac{\beta_{1t}}{2\alpha_t(1-\beta_{1t})} ||V_t^{1/2p}(x_t-x^*)||^2]
  \end{align*}
\end{lemma}

\subsubsection*{Proof of Lemma \ref{all}:}
We begin with the following observation:

$$x_{t+1}=\Pi_{\mathscr{F},V_t^{1/p}} (x_t-\alpha_tV_t^{-1/p}m_t)=\min_{x \in \mathscr{F}}||V_t^{1/2p} (x-(x_t-\alpha_tV_t^{-1/p}m_t))||$$

Using Lemma 4 in  \cite{j.2018on} with $u_1=x_{t+1}$ and $u_2=x^*$, we have the following:
\begin{align*}
||V_t^{1/2p}(x_{t+1}-x^*)||^2 &\leq ||V_t^{1/2p}(x_{t}-\alpha_t V_t^{-1/p}m_t-x^*)||^2\\
&=||V_t^{1/2p}(x_t-x^*)||^2+\alpha^2_t||V_t^{-1/2p}m_t||^2-2\alpha_t\langle m_t,x_t-x^*\rangle\\
&=||V_t^{1/2p}(x_t-x^*)||^2+\alpha^2_t||V_t^{-1/2p}m_t||^2\\
&\hspace*{0.5in}-2\alpha_t\langle \beta_{1t}m_{t-1}+(1-\beta_{1t})g_t,x_t-x^*\rangle
\end{align*}

Rearranging the above inequality, we have

\begin{align*}
\langle g_t,x_t-x^* \rangle\leq \frac{1}{2\alpha_t(1-\beta_{1t})}[ ||V_t^{1/2p}&(x_t-x^*)||^2-||V_t^{1/2p}(x_{t+1}-x^*)||^2]\\
&+\frac{\alpha_t}{2(1-\beta_{1t})}||V_t^{-1/2p}m_t||^2-\frac{\beta_{1t}}{1-\beta_{1t}}\langle  m_{t-1},x_t-x^*\rangle\\
\leq \frac{1}{2\alpha_t(1-\beta_{1t})}[ ||V_t^{1/2p}&(x_t-x^*)||^2-||V_t^{1/2p}(x_{t+1}-x^*)||^2]\\
 & +\frac{\alpha_t}{2(1-\beta_{1t})}||V_t^{-1/2p}m_t||^2+\frac{\alpha_t\beta_{1t}}{2(1-\beta_{1t})}||V_t^{-1/2p}m_{t-1}||^2 \\
 &+\frac{\beta_{1t}}{2\alpha_t(1-\beta_{1t})} ||V_t^{1/2p}(x_t-x^*)||^2
\end{align*}

The second inequality follows from simple application of Cauchy-Schwarz and Young’s inequality.
We now use the standard approach of bounding the regret at each step using convexity of the function $f_t$ in the following manner:
\begin{align*}
&\sum_{t=1}^{T} f_t(x_t)-f_t(x^*)\leq \sum_{t=1}^{T} \langle g_t,x_t-x^* \rangle\\
&\leq \sum_{t=1}^{T} [\frac{1}{2\alpha_t(1-\beta_{1t})}( ||V_t^{1/2p}(x_t-x^*)||^2-||V_t^{1/2p}(x_{t+1}-x^*)||^2)+\frac{\alpha_t}{2(1-\beta_{1t})}||V_t^{-1/2p}m_t||^2\\
&+\frac{\beta_{1t}}{2(1-\beta_{1t})}\alpha||V_t^{-1/2p}m_{t+1}||^2 +\frac{\beta_{1t}}{2\alpha_t(1-\beta_{1t})} ||V_t^{1/2p}(x_t-x^*)||^2]
\end{align*}
This completes the proof of Lemma \ref{all}.

Base on this Lemma, we are going to find the corresponding upper bound for each term in the above regret bound inequality.

For the first term $\sum_{t=1}^{T} [\frac{1}{2\alpha_t(1-\beta_{1t})}( ||V_t^{1/2p}(x_t-x^*)||^2-||V_t^{1/2p}(x_{t+1}-x^*)||^2)$, we have Lemma \ref{semi-positive part}.

\begin{lemma}\label{semi-positive part}
When $B_t/t$ is non-increasing, then $V_t/\alpha^2_t-V_{t-1}/\alpha^2_{t-1}$ is semi-positive, and
$$\sum_{t=1}^{T} \frac{1}{2\alpha_t(1-\beta_{1t})}( ||V_t^{1/2p}(x_t-x^*)||^2-||V_t^{1/2p}(x_{t+1}-x^*)||^2)\leq \frac{T^{1/p}}{2(1-\beta_{1})} \frac{V_t^{1/p}}{\alpha} D_{\infty}^2$$

\end{lemma}

\subsubsection*{Proof of Lemma \ref{semi-positive part}:}
  \begin{align*}
  \frac{V_t}{\alpha_t^p}
  =& \frac{t}{\alpha^p}\sum_{j=1}^t \Pi_{k=1}^{t-j} \beta_{2,t-k+1}(1-\beta_{2,j})g_j^p\\
  =& \frac{t}{\alpha^p} \sum_{j=1}^t \frac{B_{t-1}}{B_t} \ldots \frac{B_{j}}{B_{j+1}}\frac{B_{j}-B_{j-1}}{B_{j}} g_j^p \\
  =& \frac{t}{B_t \alpha^p} \sum_{j=1}^t b_j g_j^p
  \geq \frac{t-1}{B_{t-1} \alpha^2} \sum_{j=1}^{t-1} b_j g_j^p\\
  =&  \frac{V_{t-1}}{\alpha_{t-1}^p}
  \end{align*}
  
  which means $V_t/\alpha^2_t-V_{t-1}/\alpha^2_{t-1}$ is semi-positive.

\begin{align*}
    &\sum_{t=1}^{T} \frac{1}{2\alpha_t(1-\beta_{1t})}( ||V_t^{1/2p}(x_t-x^*)||^2-||V_t^{1/2p}(x_{t+1}-x^*)||^2)\\
    \leq & \frac{1}{2(1-\beta_{1})} \sum_{t=1}^{T} \frac{1}{\alpha_t}( ||V_t^{1/2p}(x_t-x^*)||^2-||V_t^{1/2p}(x_{t+1}-x^*)||^2)\\
    \leq&       \frac{1}{2(1-\beta_{1})}(  ||V_1^{1/2p}(x_1-x^*)||^2 - ||V_T^{1/2p}(x_{T+1}-x^*)||^2)
    \\
    +& \frac{1}{2(1-\beta_{1})} \sum_{t=2}^{T} (\frac{V_t^{1/p}}{\alpha_t}-\frac{V_{t-1}^{1/p}}{\alpha_{t-1}})(x_t-x^*)^2\\
    \leq &       \frac{1}{2(1-\beta_{1})}  ||V_1^{1/2p}||^2 D_{\infty}^2
    + \frac{1}{2(1-\beta_{1})} \sum_{t=2}^{T} (\frac{V_t^{1/p}}{\alpha_t}-\frac{V_{t-1}^{1/p}}{\alpha_{t-1}})D_{\infty}^2\\
    =&\frac{T^{1/p}}{2(1-\beta_{1})} \frac{V_t^{1/p}}{\alpha} D_{\infty}^2
    \end{align*}
The third inequation use the knowledge that $V_t/\alpha^2_t-V_{t-1}/\alpha^2_{t-1}\geq 0$.

This completes the proof of Lemma \ref{semi-positive part}.

For the second and the third terms in Lemma \ref{all}, we have Lemma \ref{new one,p>1}.

\begin{lemma}\label{new one,p>1}
  
$$
\frac{\alpha_t}{2(1-\beta_{1t})}||V_t^{-1/2p}m_t||^2+\frac{\alpha_t\beta_{1t}}{2(1-\beta_{1t})}||V_t^{-1/2p}m_{t-1}||^2 \leq  \frac{p}{2(p-1)}\frac{\alpha(1+\beta_1)}{(1-\beta_1)^3} S_T^{\frac{p-1}{p}}  (\frac{B_T}{T b_T^p})^{\frac{1}{p}} $$

\end{lemma}

\subsubsection*{Proof of Lemma \ref{new one,p>1}}

For the second term in Lemma \ref{all}  :
  \begin{align*}
  &\sum_{t=1}^{T} \alpha_t ||V_t^{-1/2p}m_t||^2 \\
  =&\sum_{t=1}^{T-1}\alpha_t ||V_t^{-1/2p}m_t||^2+ \alpha_T\sum_{i=1}^{d}\frac{m_{T,i}^2}{v_{T,i}^{1/p}}\\
  \leq &\sum_{t=1}^{T-1}\alpha_t ||V_t^{-1/2p}m_t||^2 + \alpha_T\sum_{i=1}^{d} \frac{(\sum_{j=1}^{T}(1-\beta_{1j})\Pi_{k=1}^{T-j}\beta_{1(T-k+1)}g_{j,i})^2 )}{(\sum_{j=1}^{T}\Pi_{k=1}^{T-j}\beta_{2(T-k+1)}(1-\beta_{2j})g_{j,i}^2)^{1/p}}\\
  \leq &  \sum_{t=1}^{T-1}\alpha_t ||V_t^{-1/2p}m_t||^2 \\
 & \hspace*{0.2in} + \alpha_T\sum_{i=1}^{d} \frac{   (\sum_{j=1}^{T}(1-\beta_{1j})\Pi_{k=1}^{T-j}\beta_{1(T-k+1)} )(\sum_{j=1}^{T}(1-\beta_{1j})\Pi_{k=1}^{T-j}\beta_{1(T-k+1)}g_{j,i}^2 )   }{(\sum_{j=1}^{T}\Pi_{k=1}^{T-j}\beta_{2(T-k+1)}(1-\beta_{2j})g_{j,i}^2)^{1/p}}\\
  \leq & \sum_{t=1}^{T-1}\alpha_t ||V_t^{-1/2p}m_t||^2 + \alpha_T\sum_{i=1}^{d} \frac{   (\sum_{j=1}^{T}\beta_{1}^{T-j})(\sum_{j=1}^{T}\beta_{1}^{T-j}g_{j,i}^2 )   }{(\sum_{j=1}^{T}\Pi_{k=1}^{T-j}\beta_{2(T-k+1)}(1-\beta_{2j})g_{j,i}^2)^{1/p}}\\
  \leq & \sum_{t=1}^{T-1}\alpha_t ||V_t^{-1/2p}m_t||^2 + \frac{\alpha_T}{1-\beta_1}  \sum_{i=1}^{d} \frac{ \sum_{j=1}^{T}\beta_{1}^{T-j}g_{j,i}^2  }{(\sum_{j=1}^{T}\Pi_{k=1}^{T-j}\beta_{2(T-k+1)}(1-\beta_{2j})g_{j,i}^2)^{1/p}}\\
  \leq & \sum_{i=1}^{d} \sum_{t=1}^{T}\frac{\alpha_t}{1-\beta_1}   \frac{ \sum_{j=1}^{t}\beta_{1}^{t-j}g_{j,i}^2  }{(\sum_{j=1}^{t}\Pi_{k=1}^{t-j}\beta_{2(T-k+1)}(1-\beta_{2j})g_{j,i}^2)^{1/p}}\\
  \end{align*}

 For the third term of in Lemma \ref{all}  :
  \begin{align*}
  &\sum_{t=1}^{T} \alpha_t ||V_t^{-1/2p}m_{t-1}||^2 \\
  =&\sum_{t=1}^{T-1}\alpha_t ||V_t^{-1/2p}m_{t-1}||^2+ \alpha_T\sum_{i=1}^{d}\frac{m_{T-1,i}^2}{v_{T,i}^{1/p}}\\
  \leq &\sum_{t=1}^{T-1}\alpha_t ||V_t^{-1/2p}m_{t-1}||^2 + \alpha_T\sum_{i=1}^{d} \frac{(\sum_{j=1}^{T-1}(1-\beta_{1j})\Pi_{k=1}^{T-1-j}\beta_{1(T-k)}g_{j,i})^2 )}{(\sum_{j=1}^{T}\Pi_{k=1}^{T-j}\beta_{2(T-k+1)}(1-\beta_{2j})g_{j,i}^2)^{1/p}}\\
  \leq & \sum_{t=1}^{T-1}\alpha_t ||V_t^{-1/2p}m_{t-1}||^2 + \alpha_T\sum_{i=1}^{d} \frac{   (\sum_{j=1}^{T}\beta_{1}^{T-1-j})(\sum_{j=1}^{T-1}\beta_{1}^{T-1-j}g_{j,i}^2 )   }{(\sum_{j=1}^{T}\Pi_{k=1}^{T-j}\beta_{2(T-k+1)}(1-\beta_{2j})g_{j,i}^2)^{1/p}}\\
  \leq & \sum_{t=1}^{T-1}\alpha_t ||V_t^{-1/2p}m_{t-1}||^2 + \frac{\alpha_T}{1-\beta_1}  \sum_{i=1}^{d} \frac{ \sum_{j=1}^{T-1}\beta_{1}^{T-1-j}g_{j,i}^2  }{(\sum_{j=1}^{T}\Pi_{k=1}^{T-j}\beta_{2(T-k+1)}(1-\beta_{2j})g_{j,i}^2)^{1/p}}\\
  \leq & \sum_{i=1}^{d} \sum_{t=1}^{T}\frac{\alpha_t}{1-\beta_1}   \frac{ \sum_{j=1}^{t-1}\beta_{1}^{t-1-j}g_{j,i}^2  }{(\sum_{j=1}^{t}\Pi_{k=1}^{t-j}\beta_{2(T-k+1)}(1-\beta_{2j})g_{j,i}^2)^{1/p}}\\
  \end{align*}
  
  What's more:
  
  \begin{align*}
  &\sum_{t=1}^{T}  \frac{\alpha_t}{1-\beta_1} \frac{\sum_{j=1}^t\beta_1^{t-j}g_{j}^p}{ [\sum_{j=1}^t \Pi_{k=1}^{t-j}\beta_{2,t-k+1}(1-\beta_{2,j})g_{j}^p]^{\frac{1}{p}}}
  =
  \sum_{t=1}^{T}  \frac{\alpha}{1-\beta_1} \frac{\sum_{j=1}^t\beta_1^{t-j}g_{j}^p}{ [\sum_{j=1}^t \frac{t}{B(t)}b_j g_{j}^p]^{\frac{1}{p}}}      \\
  \le&\sum_{t=1}^{T}  \frac{\alpha}{1-\beta_1} (\frac{B_t}{t})^{\frac{1}{p}} \sum_{j=1}^t \frac{\beta_1^{t-j}g_{j}^p}{(\sum_{k=1}^j b_k g_{k}^p)^{\frac{1}{p}}}
  = \sum_{j=1}^T\sum_{t=j}^T \frac{\beta_1^{t-j}g_{j}^p}{(\sum_{k=1}^j b_k g_{k}^p)^{\frac{1}{p}} }(\frac{B_t}{t})^{\frac{1}{p}}  \frac{\alpha}{1-\beta_1}\\
  \leq& \sum_{j=1}^T\sum_{t=j}^T \frac{\beta_1^{t-j}g_{j}^p}{(\sum_{k=1}^j b_k g_{k}^p)^{\frac{1}{p}} }(\frac{B_j}{j})^{\frac{1}{p}}  \frac{\alpha}{1-\beta_1}\\
  \leq& \frac{\alpha}{(1-\beta_1)^2} \sum_{j=1}^T \frac{g_{j}^p}{(\sum_{k=1}^j b_k g_{k}^p)^{\frac{1}{p}}}  (\frac{B_j}{j})^{\frac{1}{p}}
  \end{align*}

  The first inequality comes from $\sum_{k=1}^j b_k g_{k}^p\le \sum_{k=1}^t b_k g_{k}^p$.
  The second inequality comes from  that $B_t/t$ is non-increasing with respect to $t$. The last inequality follows from then inequality
  $\sum_{t=j}^{T} \beta_{1}^{t-j}\le 1/(1-\beta_{1})$.
  
  Similar to the proof above:
  
  \begin{align*}
  &\sum_{t=1}^{T}  \frac{\alpha_t}{1-\beta_1} \frac{\sum_{j=1}^{t-1}\beta_1^{t-1-j}g_{j}^p}{ [\sum_{j=1}^t \Pi_{k=1}^{t-j}\beta_{2,t-k+1}(1-\beta_{2,j})g_{j}^p]^{\frac{1}{p}}}
  =
  \sum_{t=1}^{T}  \frac{\alpha}{1-\beta_1} \frac{\sum_{j=1}^{t-1}\beta_1^{t-1-j}g_{j}^p}{ [\sum_{j=1}^t \frac{t}{B(t)}b_j g_{j}^p]^{\frac{1}{p}}}      \\
  \le&\sum_{t=1}^{T}  \frac{\alpha}{1-\beta_1} (\frac{B_t}{t})^{\frac{1}{p}} \sum_{j=1}^{t-1} \frac{\beta_1^{t-1-j}g_{j}^p}{(\sum_{k=1}^j b_k g_{k}^p)^{\frac{1}{p}}}
  = \sum_{j=1}^{T-1}\sum_{t=j+1}^T \frac{\beta_1^{t-1-j}g_{j}^p}{(\sum_{k=1}^j b_k g_{k}^p)^{\frac{1}{p}} }(\frac{B_t}{t})^{\frac{1}{p}}  \frac{\alpha}{1-\beta_1}\\
  \leq& \sum_{j=1}^{T-1}\sum_{t=j+1}^T \frac{\beta_1^{t-1-j}g_{j}^p}{(\sum_{k=1}^j b_k g_{k}^p)^{\frac{1}{p}} }(\frac{B_j}{j})^{\frac{1}{p}}  \frac{\alpha}{1-\beta_1}\\
  \leq& \frac{\alpha}{(1-\beta_1)^2} \sum_{j=1}^{T-1} \frac{g_{j}^p}{(\sum_{k=1}^j b_k g_{k}^p)^{\frac{1}{p}}}  (\frac{B_j}{j})^{\frac{1}{p}}
  \end{align*}

  Let $S_j = \sum_{k=1}^j b_k g_{k}^p$,
  \begin{align*}
  &\sum_{j=1}^T \frac{g_{j}^p}{(\sum_{k=1}^j b_k g_{k}^p)^{\frac{1}{p}}}  (\frac{B_j}{j})^{\frac{1}{p}}
  = \sum_{j=1}^T\frac{g_{j}^p}{S_j^{\frac{1}{p}}}(\frac{B_j}{j})^{\frac{1}{p}} \\
  \leq& \sum_{j=1}^T g_j^p    (\frac{B_j}{j})^\frac{1}{p}    \frac{p}{p-1} \frac{ ( S_j^{\frac{p-1}{p}}-S_{j-1}^{\frac{p-1}{p}} )   }{  S_j-S_{j-1}   }
  = \frac{p}{p-1}  \sum_{j=1}^T (  S_j^{\frac{p-1}{p}}-S_{j-1}^{\frac{p-1}{p}}    )(\frac{B_j}{j b_j^p})^\frac{1}{p}\\
  =&\frac{p}{p-1}  S_T^{\frac{p-1}{p}}  (\frac{B_T}{T b_T^p})^{\frac{1}{p}} +   \frac{p}{p-1} \sum_{j=1}^{T-1} [-(\frac{B_{j+1}}{(j+1) b_{j+1}^p})^{\frac{1}{p}} +(\frac{B_j}{j b_j^p})^{\frac{1}{p}} ] S_{j}^{\frac{p-1}{p}} \\
  \leq&  \frac{p}{p-1}  S_T^{\frac{p-1}{p}}  (\frac{B_T}{T b_T^p})^{\frac{1}{p}}
  \end{align*}
  
  The first inequality comes from Lemma \ref{when p>1} when $p>1$. The last inequality comes from the second constraint, which tells us that $B_j/(j b_j^p)$ is non-decreasing with respect to $j$. This completes the proof of the lemma.

This finally completes the proof of Lemma \ref{new one,p>1}.

To complete Lemma \ref{new one,p>1}, we need to finally prove the next Lemma.

  \begin{lemma}\label{when p>1}
    $$\frac{1}{S_j^{\frac{1}{p}}} \le \frac{p}{p-1} \frac{ ( S_j^{\frac{p-1}{p}}-S_{j-1}^{\frac{p-1}{p}} )   }{   S_j-S_{j-1}   }$$
  \end{lemma}

\subsubsection*{Proof of Lemma \ref{when p>1}}
    When $p>1$,$s>0$,$x\geq 0$
    \begin{align*}
    &1\le \frac{p}{p-1} [1- \frac{1}{p} (\frac{s}{s+x})^{\frac{p-1}{p}} ] \\
    \Rightarrow &x\le \frac{p}{p-1} [ (s+x) - s^{\frac{p-1}{p}}(s+x)^{\frac{1}{p}} ]\\
    \Rightarrow & \frac{1}{(s+x)^{\frac{1}{p}}} \le \frac{p}{p-1}  \frac{ ( (s+x)^{\frac{p-1}{p}}-s^{\frac{p-1}{p}} )   }{ x }\\
    \Rightarrow&\frac{1}{S_j^{\frac{1}{p}}} \le \frac{p}{p-1} \frac{ ( S_j^{\frac{p-1}{p}}-S_{j-1}^{\frac{p-1}{p}} )   }{   S_j-S_{j-1}   }\\
    \end{align*}
This completes the proof of Lemma \ref{when p>1}.

Finally, back to our theorem, using the inequalities in Lemma \ref{new one,p>1} and Lemma \ref{semi-positive part} to substitute the first three terms in Lemma \ref{all} completes the proof of our theorem.

\end{appendices}
\bibliographystyle{named}
\bibliography{ijcai19}

\end{document}